\documentclass{article}
\usepackage[utf8]{inputenc}
\usepackage{graphicx}
\usepackage{amsmath}
\usepackage{amsfonts}
\usepackage{url}

\usepackage{amsthm}

\newtheorem{theorem}{Theorem}[section]

\newtheorem{proposition}[theorem]{Proposition}

\theoremstyle{definition}
\newtheorem{definition}{Definition}[section]
\newtheorem*{remark}{Remark}

\title{Kullback-Leibler divergence between quantum distributions, and its upper-bound}
\author{Vincenzo Bonnici}
\date{December 2020}

\begin{document}

\maketitle

\begin{abstract}
This work presents an upper-bound to value that the Kullback–Leibler (KL) divergence can reach for a class of probability distributions called quantum distributions (QD). The aim is to find a distribution $U$ which maximizes the KL divergence from a given distribution $P$ under the assumption that $P$ and $U$ have been generated by distributing a given discrete quantity, a quantum. 
Quantum distributions naturally represent a wide range of probability distributions that are used in practical applications. Moreover, such a class of distributions can be obtained as an approximation of any probability distribution.
The retrieving of an upper-bound for the entropic divergence is here shown to be possible under the condition that the compared distributions are quantum distributions over the same quantum value, thus they become {\it comparable}. Thus, entropic divergence acquires a more powerful meaning when it is applied to comparable distributions.
This aspect should be taken into account in future developments of divergences. 
The theoretical findings are used for proposing a notion of normalized KL divergence that is empirically shown to behave differently from already known measures.
\end{abstract}


\section{Introduction}
The Kullback-Leilber divergence (KL), also called entropic divergence, is a widely used measure for comparing two discrete probability distributions \cite{kullback1951information}. 
Such a divergence is derived from the notion of entropy, and it aims at evaluating the amount of information that is gained by switching from one distribution to another. 
The applications of the divergence range in several scientific area, for example, for testing random variables \cite{arizono1989test, li2008testing, belov2010automatic}, for selecting the right sample size \cite{clarke1999asymptotic}, for optimizing sampling in bioinformatics \cite{lin2007information} or for analysing magnetic resonance imagines \cite{volkau2006extraction}.
However, it has two important properties that act as a limitation to its applicability.
It can not be used as a metric because it is not symmetric, in fact, $KL(P||Q) \neq KL(Q||P)$, being $P$ and $Q$ two probability distributions.
Moreover, its value is 0 if equal distributions are compared, but it is shown to not have an upper-bound to its possible value. One of the reasons is because it results in an infinite divergence if the probability of a specific event is equal to 0 in $Q$ but is greater than 0 in $P$. However, even if infinite divergence is discarded, an upper-bound for the entropic divergence has not be established. 

The search for bounded divergences is an important topic in Information Theory and some attempts have been done in the past year.  For example, the main goal of the so-called Jensen-Shannon divergence (JS) \cite{lin1991divergence} is to provide a notion of symmetric divergence, but it is also shown to be upper bounded by the value 1 if the base of the used logarithm is 2.
Its is a metric but its values are not uniformly distributed within the range $[0 \dots 1]$, how it is empirically shown in this study.
Kullback-Leibler and Jensen-Shannon measures are in the class of $f$-divergences \cite{renyi1961measures} which aim at representing the divergence as an average of the odds ratio
 weighted by a function $f$.
Each divergence has specific meaning and behaviours, and the relations among different types of $f$-divergence is a well-studied topic \cite{sason2016f}. 
The Hellinger distance \cite{hellinger1909neue} is one of the most used measures among the $f$-divergences, together with Kl and JS. It avoids infinite divergences by definition and it is bounded between 0 and 1.
\\

The present work introduces a new class of discrete probability distributions, called quantum distributions. The name comes from the fact that the probabilities reported by such class of distributions are formed by quanta of probability.
The final aim of the present study is to show that, given a quantum distribution $P$, there exists another quantum distribution $U$ which maximizes the entropic divergence from $P$. Thus, for each distribution $P$, an upper-bound to the divergence form $P$ can be obtained by constructing $U$. The assumptions are that infinite divergences must be avoidable and that the two distributions must be formed by distributing a given discrete quantity, namely, they must be formed by the same quantum. This last property highlights an important previously unaccounted aspect of the KL divergence. In fact, since such a bound can only be assessed under this condition, KL divergence should only be applied between probability distributions formed by the same quantum.
These theoretical results allow the introduction of a notion of entropic divergence that is normalized in the range $[0 \dots 1]$, independently from the base of the used logarithm. Such a measure is compared with the more commonly used notions of divergence, and distance, between distributions by showing that it behaves in a very specific way. Besides, it is empirically shown that its values are better distributed in the range $[0 \dots 1]$ w.r.t. the compared measures.

\section{Quantum distributions}
A {\it finite} (thus discrete) {multiplicity distribution} is defined as a function $f$ which distributes a given discrete quantity $M$ to a finite set $C$ of $|C|$ distinct cells.
Thus, $\sum_{c \in C}f(c) = M$.
This class of distributions is often represent via Ferrers diagrams \cite{pemmaraju2003computational}, in which the distributed quantity is a finite set of $M$ dots that are assigned to cells.
A multiplicity distribution is commonly transformed into discrete probability (frequency) distributions by converting it to a distribution such that the sum of its outcomes equals 1. Thus, a finite discrete probability distribution $P$ is obtained by diving the assigned quantities for the total quantity, namely $P(c \in C) = \frac{f(c)}{M}$.
In this context, the notion of {\it quantum} relates to the fact that a distribution is defined on a (finite) discrete domain and that the assigned values are formed by quanta, namely discrete unitary pieces of information.

\begin{definition}[Quantum distribution]
A quantum (probability) distribution (QD) is a finite discrete probability distribution which assigns a probability value to each of the $n$ values of a variable. The probability values are positive, non-zero multiples of the fraction $1/M$, called the quantum of the distribution, for a given $M \in \mathbb{N}$.
The value $n$ is also called the cardinality of the distribution.
\end{definition}

It has to be noticed that since quantum distributions are probability distributions, the sum of the assigned probabilities must equal 1.
Moreover, what is defined here is a special type of probability distributions. In fact, in general, it is not required that a probability distribution is sourced by a discrete quantity $M$ distributed over a finite set of cells. 
Such a type of distribution is of great importance in the field of Computer Science, where probabilities are estimated by looking at frequencies calculated from discrete quantities, for example for representing biological information \cite{pinello2011motif, manca2013infobiotics, zambelli2018rnentropy}. 
However, it can be easily shown that the class of quantum distribution covers all the possible discrete finite probability distributions. In fact, for distributions such that assigned probabilities are rational numbers, they can always be re-scaled by setting the quantum value as 1 divided by the common denominator of the assigned probabilities. The rest of discrete finite probability distributions can be approximated by using an arbitrarily small epsilon for their discretization.

\begin{remark}
For each multiplicity distribution $D$ there exists a quantum distribution $P$, and vice versa.
\end{remark}
In fact, given a multiplicity distribution it can always be converted to a quantum distribution by dividing the assigned values to their sum. Vice versa, a quantum distribution can be represented as a function which assigns values that are an integer multiple (a multiplicity) of the quantum. 

Because of the strict relation between frequency/probability and multiplicity distributions, from now to below and without loss of generality, the assigned probability values, $\frac{f(c)}{M}$, will be interchanged with their multiplicity/integer-quantity counterpart, $f(c)$, depending on the purpose of the context in which they are recalled. Similarly, Ferret diagrams and their dot-based representation are used for representing this type of distribution.

\begin{remark}
Two distributions $P$ and $Q$, defined on the same domain $C$, are considered equal, thus not distinct, if $\forall c \in C \Rightarrow P(c) = Q(c)$.
\end{remark}

\begin{proposition}
\label{prop:orderedcount}
The total number of distinct, thus not equal, quantum distributions that can be formed by arranging a quantity $M$ in $n$ distinct cells is ${M - 1}\choose{M - n}$.
\end{proposition}
\begin{proof}

Given a set $S$ of $x$ elements, the number of $y$-combinations, namely the number of subsets of $y$ distinct elements of $S$, if given by ${x}\choose{y}$.
The number of $y$-combinations with repetitions, namely the number of sequences of $y$ non necessarily distinct elements of $S$, is given by ${x+y-1}\choose{y}$ \cite{feller2008introduction}.

Quantum distributions require that at each cell a minimum value of $1/M$ must be assigned. 
Switching from quantum to multiplicity distributions, it implies that a quantity of $n$ elements, out of $M$, does not participate to the arrangement process, since a fixed minimum value of 1 is assigned to each cell. Thus, the quantity that is arranged equals $M-n$.
Each dot must be assigned to a given cell, and no dot can remain unassigned. Thus, the arrangement process can be seen as an assignment of one specific cell to each of the $M-n$ dots by allowing a cell to be assigned to multiple dots.
Compared to the classical combinatorial problems, we are not assigning dots to cells but, instead, we are assigning cells to dots.
Thus, it equals to count the number of $(M-n)$-combinations with repetitions of a set of $n$ elements, that is given by
\begin{equation}
\binom{M -n +n -1}{M - n} = \binom{M-1}{M-n}
\end{equation}.
\end{proof}

The present study aims at showing that for each of these distributions 
there exists another distribution that maximizes the value of the entropic divergence. The proof of it, that is given in the next section, requires that the elements of the domain must be ordered according to the values assigned to them.

\begin{definition}[Ordered quantum distribution]
\label{def:oqd}
Given a quantum distribution $P$, and ordered quantum distribution (OQD) is obtained by assigning an integer index $i$, with $1 \leq i \leq |C|$, to each domain element $c \in C$ such that $P(c_i) \geq P(c_{i+1})$. $P(c_i)$ is also referred to as $P_i$.
\end{definition}
It has to be noticed that Definition \ref{def:oqd} is based on a monotonically decreasing order but, without loss of generality, a monotonically increasing order can be used too. Besides, in what follows, the greatest value of the distribution is considered to be placed on the leftmost position of it, and, consequently, the lowest value is considered to be placed on the rightmost position of the distribution.

\begin{remark}
\label{def:orderedequivalence}
Two ordered distributions $P$ and $Q$, defined on the same domain $C$, are equal, thus not distinct, if $\forall i : 1 \leq i \leq |C| \Rightarrow P_i = Q_i$.
\end{remark}

Multiple unordered distributions may produce the same ordered distribution leading them to belong to the same class of equivalence that is defined by such a shared ordered output.
Formally, $\mathbb{Q}_{M,n}$ is defined as the complete set of QDs that can be formed by arranging a quantity of $M$ into $n$ cells.
$\mathbb{O}_{M,n}$ is defined as the complete set of OQDs that can be formed by arranging a quantity of $M$ into $n$ cells.
Then, the function which transforms an unordered QD into an ordered QD, $ord : \mathbb{Q}_{M,n} \mapsto \mathbb{O}_{M,n}$, can be shown to be a surjective function.
Thus, each class of equivalence if represented by a given distribution  $O \in \mathbb{O}_{M,n}$, and it is 
formed by a subset of $\mathbb{Q}_{M,n}$, referred to as 
$\mathbb{Q}^O_{M,n}$, such that $\forall P \in \mathbb{Q}^O_{M,n} : ord(P) = O$.

We are interested in counting the number of classes, which also equals the number of distinct ordered distributions.
\begin{proposition}
The total number of distinct ordered distributions that can be formed by arranging a quantity of $M$  in $n$ cells equals the number of partitions of the integer $M$ for representing it as a sum of $n$ integer addends.
\end{proposition}
\begin{proof}
Similarly to unordered distributions (see Proposition \ref{prop:orderedcount}), ordered distributions assign a minimum quantity of $1$ to each cell. 
The number of partitions of an integer $x$ for representing it as a sum of $y$ addends, which value can not be 0, can be obtained by the recursive formula $p_y(x) = p_y(x - y) + p_{y-1}(x - 1)$, with $p_y(x) = 0$ if $y>x$ and $p_y(0) = 0$ \cite{stanley2011enumerative}.
Thus, we can use the formula to evaluate the number of ordered distributions by setting $x = M$, because it is the total arranged quantity, and $y = n$, because we want to represent such an integer as a sum of exactly $n$ non-zero addends (namely, non-empty cells).
\end{proof}

The search for a maximum value of the KL divergence between two quantum distributions $P$ and $Q$ (presented in the next section) is based on the fact that both distributions must be formed by the same quantum value. However, there are plenty of practical situations where this assumption is not verified, and the two distributions need to be transformed into two {\it comparable} distributions before the calculation of the divergence.
\begin{proposition}
Given two quantum distributions, $P$ and $Q$, formed by two different quanta, $1/M_P$ and $1/M_Q$ respectively, they can always be transformed into two distributions formed by the same quantum.
\end{proposition}
\begin{proof}
Since $M_P$ and $M_Q$ are two positive integer numbers, the least common multiple (lcm) between them can be used for re-scaling the two distributions such that they will be formed by the same quantum value. 
The new distributions are formed by the same quantum that is $M = 1/lcm(M_P,M_Q)$.
The values of these distributions are always in the form $x/M$, being $x$ a positive integer, thus the values of the new distributions can be re-scaled as $  x(M/M_Q) / M$.
It is trivial to show that the new distributions maintain their status of quantum distribution.
\end{proof}

\section{Upper-bound of the entropic divergence}

Given two probability distributions the entropic divergence, also called the Kullback–Leibler (KL) divergence from the authors who discovered it \cite{kullback1951information}, aims at measuring the information gain from one distribution to another. For two probability distributions, $P$ and $Q$, that are defined on the same domain $C$, the divergence of $P$ from $Q$ is defined as:
\begin{equation}
\begin{aligned}
KL(P||Q) = \sum_{c \in C}p(c) log_2 \frac{P(c)}{Q(c)}
\end{aligned}
\end{equation}.
The divergence is not symmetric, thus $KL(P||Q) \neq KL(Q||P)$, and its possible value ranges between $0$ and $+ \infty$. In fact, the divergence is $0$ if the two distributions equal in their outcomes, namely $P(c)=Q(c), \forall c \in C$. It has no upper-bound, as it has been shown by the Gibbs' inequality \cite{bremaud2012introduction}. However, such an affirmation has been shown by comparing two {\it general} distributions and by stating that the entropic divergence is a difference between the two quantities $-\sum_{c \in C}P(c)log_2 P(c)$ and $-\sum_{c \in C}P(c)log_2 Q(c)$, which implies
\begin{equation}
\begin{aligned}
-\sum_{c \in C}P(c)log_2 P(c) \leq -\sum_{c \in C}P(c)log_2 Q(c)
\end{aligned}
\end{equation},
and thus 
\begin{equation}
\begin{aligned}
KL(P||Q) = \sum_{c \in C}P(c) log_2 \frac{P(c)}{Q(c)} \geq 0
\end{aligned}
\end{equation}.

Given two positive numbers $M$ and $n$, the previous section shows that the sets $\mathbb{Q}_{M,n}$ and $\mathbb{O}_{M,n}$ are finite sets. Thus, it is trivial to show that, for each $P$ belonging to one of the two sets, there exists a distribution $U$ in $\mathbb{Q}_{M,n}$ or $\mathbb{O}_{M,n}$ for which $KL(P||U)$ is max.
Here, we are interested in finding such a distribution $U$.
It has to be pointed out that $P$ and $U$ are quantum distributions formed by the same quantum.
This assumption is of crucial importance to obtain, in practical situations, an upper bound to the divergence from a given distribution $P$.
\\

The general concept of distribution, and thus of probability distribution, is independent of a given ordering of the elements in $C$. 
In this perspective, ordered quantum distributions are used without loss of generality.
The KL formula for ordered distributions can be written as:
\begin{equation}
\begin{aligned}
KL(P||Q) = \sum_{1 \leq i \leq n}P_ilog_2 \frac{P_i}{Q_i}
\end{aligned}
\end{equation}.
It has to be pointed out that the ordering does not affect, in any way, the value of the KL divergence. 
This means that the distribution which maximizes the KL value from a given distribution $P$ also maximizes the KL for all the unordered distributions within the same class of equivalence of $P$ defined in the previous section.
Thus, the goal is to define the shape of the distribution $U$ which maximized the entropic divergence to $P$.

In order to avoid infinite divergences, it is required that the compared distributions, $P$ and $U$, must be defined on the set set $C$ and that for each element the two distributions are non-zero valued, namely $P_i > 0$ and $U_i > 0$ for $1 \leq i \leq n$. 
This constraint, together with the discretization of the quantity that is distributed to the cells, implies that at each cell at least a quantity equal to 1 is assigned, that is $P_i, U_i \geq 1/M$ for every $i$.
Thus, for constructing the distribution $U$, the quantity that must be arranged is $M - n$.


The entropic divergence is a sum of terms in the form $P_i log_2( P_i / U_i )$. If $P_i < U_i$ then a negative contribution is given to the sum because of the logarithmic function, while positive contributions are given for $P_i \geq U_i$. Thus, the aim is to reduce the number of positions with negative contributions. 
Each term is mediated by the $P_i$ factor, thus, it is preferable to assign positive contributions to the greatest $P_i$ values. On the contrary, negative contributions should be assigned to the smallest $P_i$ values. This means that, if $P$ is monotonically decreasing ordered (from left to right), then positive contributions should be on the left side of the distributions, and negative terms should be on the right side. 
Furthermore, the greater is $P_i$ w.r.t. $U_i$, the higher is the value of the divergence. This translates to try to increase as much as possible the difference between the greatest $P_i$ values and their corresponding $U_i$ counterparts. Of course, reducing the quantity that is assigned to the initial positions of $U$ results in increasing the quantity that is assigned to the right positions of it. 

All of these considerations lead to the intuition that the distribution that maximizes the entropic divergence is the one that minimizes the quantity assigned to positions from 1 to $n-1$, and that assigns all the remaining amount to the last position. Since the minim amount of quantity is equal to 1, then such a distribution assigns the remaining $M-n+1$ quantity to the last position $n$.
In what follows, it is shown that if $P$ is monotonically decreasing ordered then such a distributional shape maximizes the entropic divergence independently from how the quantity is distributed in $P$. This fact also implies that such maximization is independent of the ordering of $P$. It is only necessary that the quantity $M-n+1$  is assigned to the position $i$, rather than $n$, where $P_i$ is minimal. However, the ordering is helpful to prove the initial statement.

From here on, the maximizing distribution is always referred to as $U$ and any other competitor distribution is referred to as $Q$.
The proof that the entropic divergence from $U$ to $P$ is greater than the divergence from any other distribution $Q$ is split into two parts. 
Firstly, a special case is addressed, then the proof of the general case is given
\\

The special case is presented in Figure \ref{fig:case1}.
A total amount of $M=11$ elements are arranged into $n=5$ cells to compose the distributions.
As introduced above, the $P$ distribution has a monotonically decreasing order and the $U$ distributions assign a quantity of $M-n+1$ to the last cell. 
The special case is represented by the $Q$ distribution which assigns a quantity of 2 to the $(n-1)$-th position and a quantity of $M-n$ to the last position. 
For all the distributions, for every cell, a minimal quantity of 1 is assigned.
The goal is to show that:
\begin{equation}
KL(P||U) > KL(P||Q)
\end{equation}.

\begin{figure}[h]
\centering
\label{fig:case1}
\includegraphics[width=6cm]{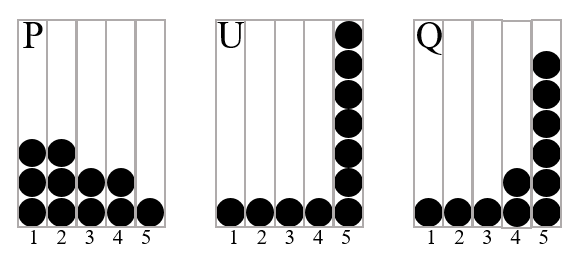}
\caption{First special case. Each element is represented as a dot that is assigned to one of the cells. A total of 11 elements are assigned to a total of 5 cells, for each of the three distributions, $P$, $U$ and $Q$ that are present in the case.}
\end{figure}

From cell 1 to cell 3, the two divergences have equal contribution, thus they differ by the last two terms. Thus, the inequality can be written as:
\begin{equation}
\begin{aligned}
P_4 log_2 \frac{P_4}{U_4} + P_5 log_2 \frac{P_5}{U_5}
>
P_4 log_2 \frac{P_4}{Q_4} + P_5 log_2 \frac{P_5}{Q_5}
\end{aligned}
\end{equation}
, that is
\begin{equation}
\begin{aligned}
P_4 log_2 P_4 - P_4 log_2 U_4 + P_5 log_2 P_5 -  P_5 log_2 U_5 \\
> \\
P_4 log_2 P_4 - P_4 log_2 Q_4 + P_5 log_2 P_5 -  P_5 log_2 Q_5 
\end{aligned}
\end{equation}
, that is
\begin{equation}
\begin{aligned}
- P_4 log_2 U_4 -  P_5 log_2 U_5 >  - P_4 log_2 Q_4 -  P_5 log_2 Q_5 \\
- \frac{2}{11} log_2 \frac{1}{11} - \frac{1}{11} log_2 \frac{5}{11} 
>
- \frac{2}{11} log_2 \frac{2}{11} - \frac{1}{11} log_2 \frac{4}{11} \\
- \frac{1}{11} log_2 5
>
- \frac{2}{11} log_2 2 - \frac{1}{11} log_2 4\\
- log_2 5 > - 2 log_2 2 - log_2 4\\
- log_2 5 > - 4
\end{aligned}
\end{equation}
, that is true.

A general proof of this special case, independently from the values of $M$ and $n$, is given in the Supplementary Materials (Prop. 2.1).
\\

Moving forward, the final goal is to show that $U$ is maximizing the divergence w.r.t any possible distribution $Q$ that is obtained by arranging the $M-n$ quantity to all the cells. 

\begin{proposition}
Let $P$ to be a OQD obtained by distributing a quantity $M$ to $n$ cells.
Let $U$ to be a QD which assigns all the free quantity $M-n$ to the $n$-th cell, in addition to the minimum quantity of $1$ to each cell.
Let $Q$ to be a QD which assigns the free quantity in a way different from $U$, in addition to the minimum quantity of $1$ to each cell.
Then, $KL(P||U) > KL(P||Q)$, independently on how the quantity is arranged in $Q$.
\end{proposition}

\begin{proof}

The initial $n-1$ cells of $U$ have a value equal to $1/M$, and the last cell has a value equal to $\frac{M-n+1}{M}$.
Instead, for what concerns $Q$, a quantity equal to $1+x_i$, for $x_i \geq 0$ , is assigned to each cell, such that $\sum_{1 \leq i \leq n}x_i = M-n$.

The following inequality must be verified:
\begin{equation}
\label{eq:case3:1}
\begin{aligned}
\Bigg(\sum_{1\leq i \leq n-1}P_i log_2 \frac{P_i}{ \frac{1}{M}} \Bigg)+ P_n log_2 \frac{ P_n }{  \frac{ M-n+1 }{ M } }
>
\sum_{1\leq i \leq n} P_i log_2 \frac{P_i}{ \frac{1+x_i}{M} } 
\end{aligned}
\end{equation},
that is
\begin{equation}
\begin{aligned}
\Bigg(\sum_{1\leq i \leq n-1}P_i log_2 \frac{P_i}{ \frac{1}{M}} \Bigg)+ P_n log_2 \frac{ P_n }{  \frac{ M-n+1 }{ M } }\\
>\\
\Bigg(\sum_{1\leq i \leq n-1} P_i log_2 \frac{P_i}{ \frac{1+x_i}{M}} \Bigg)+ P_n log_2 \frac{P_n}{ \frac{1+x_n}{M} }
\end{aligned}
\end{equation}.
The left side of the inequality is composed of a series of terms $P_i log_2 \frac{P_i}{ \frac{1}{M}}$ each of which equals $P_i log_2 P_i - P_i log_2 1 + P_i log_2 M$, and the entire inequality can be written as
\begin{equation}
\begin{aligned}
 \sum_{1\leq i \leq n-1} \Bigg(  P_i log_2 P_i - P_i log_2 1 + P_i log_2 M \Bigg)\\
 +P_n log_2 P_n  -P_n log_2 (M-n+1) +P_n log_2 M\\
>\\
 \sum_{1\leq i \leq n-1}  \Bigg( P_i log_2 P_i - P_i log_2 (x_i+1) + P_i log_2 M \Bigg)\\
 +P_n log_2 P_n  -P_n log_2 (x_n+1) +P_n log_2 M
\end{aligned}
\end{equation},
that is
\begin{equation}
\begin{aligned}
 \sum_{1\leq i \leq n-1} \Bigg(  P_i log_2 1\Bigg) - P_n log_2 (M-n+1) \\
>\\
 \sum_{1\leq i \leq n-1}  \Bigg( - P_i log_2 (x_i+1) \Bigg) - P_n log_2 (x_n+1)
\end{aligned}
\end{equation},
that is
\begin{equation}
\begin{aligned}
 -P_n log_2 (M-n+1)
>
 - \sum_{1\leq i \leq n}  P_i log_2 (x_i+1)
\end{aligned}
\end{equation}.

Since $P$ is ordered, for each position $i$ it happens that $P_i = P_{i-1} - \epsilon_i$, namely $P_{i-1} = P_i + \epsilon_i$.
The inequality can be written as
\begin{equation}
\begin{aligned}
P_n log_2 (M-n+1)
<\\
(P_n) log_2 (x_n + 1) +\\
(P_n + \epsilon_n) log_2 (x_{n-1} + 1) +\\
(P_n + \epsilon_n + \epsilon_{n-1}) log_2 (x_{n-2} + 1) +\\
\dots\\
(P_n + \epsilon_n + \dots + \epsilon_{n-n+2}) log_2 (x_{n-n+1} + 1)
\end{aligned}
\end{equation}.
The arguments of the logarithms are always greater than 1, thus the values of the logarithms are always positive.
Moreover, the factors that multiply the logarithms are always positive, because they are probabilities.
The inequality can be written as
\begin{equation}
\begin{aligned}
P_n log_2 (M-n+1)
<
n P_n \Bigg( \sum_{1 \leq i \le n} log_2(x_i + 1) \Bigg) + c
\end{aligned}
\end{equation}, 
with $c \geq 0$ and $\sum_{1 \leq i \le n} (x_i + 1) = M - n$.
By taking into account that the sum of logarithms is greater than the logarithm of the sum \cite{cover1991elements}, it is now trivial to show that the inequality is always satisfied.
\end{proof}

The previous proof is given for an ordered distribution $P$. However, the final inequality is independent from the ordering. In fact, it puts in relation the quantity $M-n+1$ (that is the one that makes $U$ the distribution of interest) with the sum of the $x_i+1$ terms independently on their position and specific value.
$P$ is ordered, and  $U$ assigns by construction the additional $M-n$ quantity to the cell where $P$ has the lowest assigned value.

The retrieving of an upper-bound for the entropic divergence is here shown to be possible under to main conditions: (i) no zero values are assigned by the two distributions; (ii) the compared distributions are quantum distributions over the same quantum value  $1/M$.
The first conditions is often ensured in practical applications, where pseudo-counts are used for avoiding infinite divergences.
The second condition emerges from this study. It states that the entropic divergence acquires a more powerful meaning when it is applied to {\it comparable} distributions. The term {\it comparable} is referred to the sharing of the same quantum value. 
This aspect should be taken into account in future developments of divergences.

\section{A notion of normalized entropic divergence}

The retrieving of the maximizing distribution is exploited in order to normalize the entropic divergence in the range $[0,1]$, both included. Given two distributions $P$ and $Q$, the normalized entropic divergence is calculated as 
\begin{equation}
\begin{aligned}
KN(P||Q) = \frac{KL(P||Q)}{KL(P||U)}
\end{aligned}
\end{equation},
where $U$ is the distribution for which the maximum entropic divergence from $P$ is obtained. 
Such a maximizing distribution is built by exploiting the results obtained in the previous section.
Namely, it distributes a minimum value of 1 to each cell, and the remaining quantity $M-n$ is assigned to the cell for which the value in $P$ is the minimum.
\\

In what follows, the proposed normalized entropic divergence is compared with the most used notions of entropic divergence, plus a measure that is highly suitable for comparing multiplicity distributions. The comparison is performed by looking at three different aspects: (i) the difference in the values that the measures output on comparing two distributions (see Section \ref{sec:comp2by2}); (ii) the spread of output values within the output range (see Section \ref{sec:odiversity}; (iii) the diversity of the measures in assigning a rank (see Section \ref{sec:diffrank}).
The relation between the measures and the properties of the compared distributions is investigated, too (see Section \ref{sec:reldistrprop}).

The investigations are empirically conducted by computationally generating the distributions.
The source code for generating the unordered and ordered distributions, together with the computational experiments, is available at the following link \url{https://github.com/vbonnici/KL-maxima}.

\subsection{Compared measures}
The proposed divergence is compare with the unnormalized one, namely $KL(P||Q)$, and with the common used symmetric divergence, also called Jensen–Shannon divergence (JS). The JS divergence is defined as 
\begin{equation}
\begin{aligned}
JSD(P,Q) = \frac{KL(P||A) + KL(Q||A)}{2}
\end{aligned}
\end{equation}
, with $A = \frac{P+Q}{2}$, and it is known to be upper-bounded by $1$ if the base of the logarithm is 2 \cite{lin1991divergence}.

Another important divergence is the Hellinger distance that is defined as 
\begin{equation}
\begin{aligned}
HE(P,Q) = \frac{1}{\sqrt{2}} \sqrt{ \sum_{ 1 \leq i \leq n} (\sqrt{P_i} - \sqrt{Q_i})^2 }
\end{aligned}
\end{equation},
and it can also be written as $HE^2(P,Q) = 1 - \sum_{1 \leq i \leq n} \sqrt{P_i Q_i}$. Important properties of such a divergence are that it implicitly avoids infinite divergences and it is bounded in the range $[0 \dots 1]$.


The generalized Jaccard similarity is a measure suitable for comparing multiplicity distributions.
It if defined as:
\begin{equation}
\label{eq:genjaccard}
\begin{aligned}
J(P,Q) = \frac{  \sum_{1 \leq i \leq n} min(P_i, Q_i)  }{ \sum_{1 \leq i \leq n} max(P_i, Q_i)  }
\end{aligned}
\end{equation}.
It can be shown that such a measure ranges from 0 to 1, both included. The minimum value is reached when the two distributions have no multiplicity in common, which means that $P_i = 0$ when $Q_i \neq 0$ and vice versa. 
It reaches the maximum value when the two distributions have equal values. 
It is a notion of similarity, therefore, it is in contrast with the meaning of the entropic divergence. Thus, for this study, it is converted as $\widetilde{J}(P,Q) = 1 - J(P,Q)$ to have it as a notion of distance.

The generalized Jaccard distance is directly applied to multiplicity distributions, while entropic divergences are applied after converting the distributions into probability/frequency distributions.

\subsection{Direct comparison of output values}
\label{sec:comp2by2}
Unordered distributions are built by using the computational procedure, then two-by-two comparisons are performed.
A scatter plot is made by using the two measures, for example, the generalized Jaccard distance and the normalized entropic divergence, for locating each two-by-two comparison.
The chart is also equipped with two histograms located aside of the axes that report the number of instances that falls within a given range of values.

\begin{figure}[h]
\centering
\includegraphics[width=0.9\linewidth]{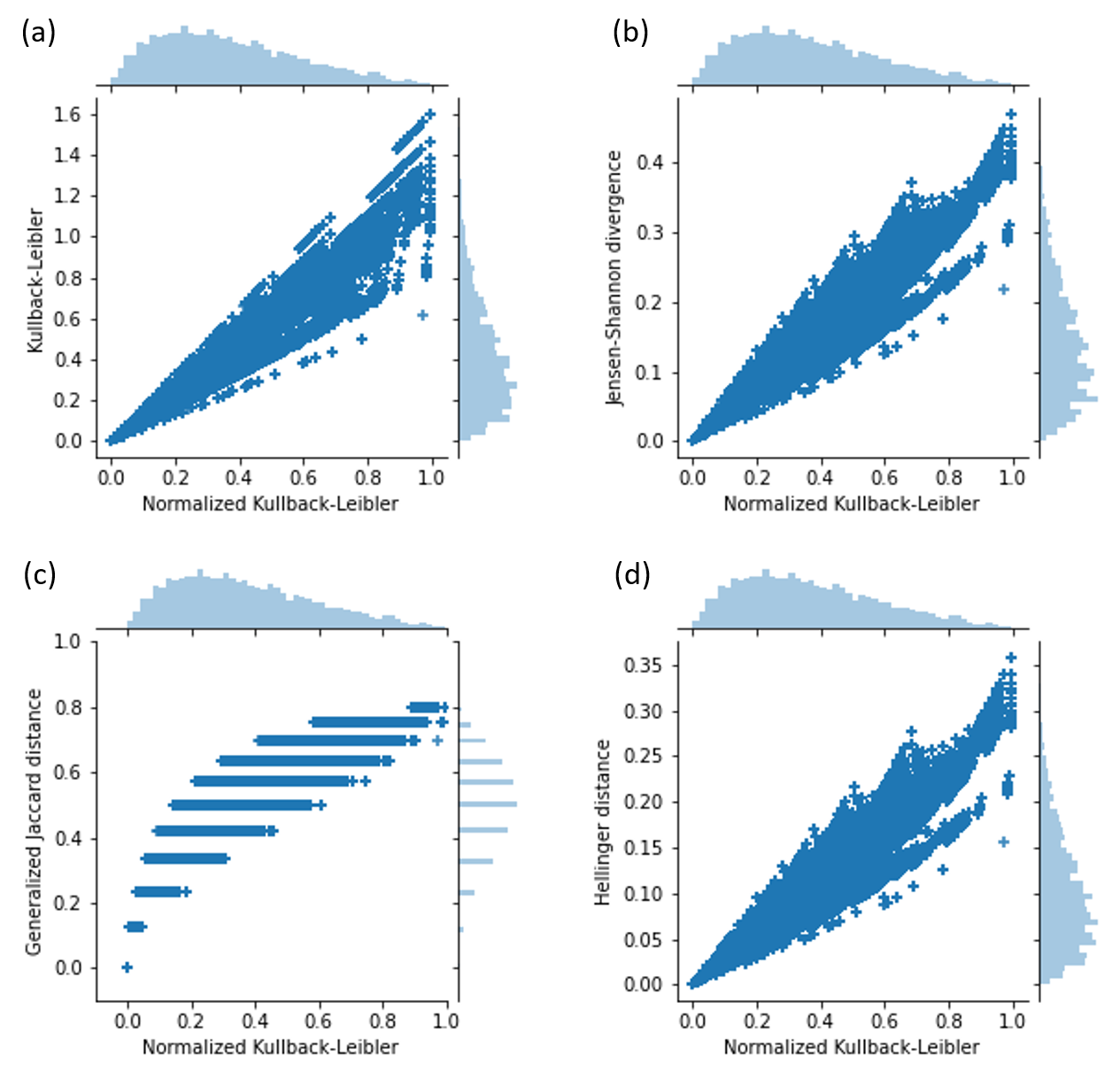}
\caption{Relation between the proposed normalized Kullback-Leibler divergence and (a) unnormalized Kullback-Leibler divergence; (b) symmetric Kullback-Leibler divergence; (c) generalized Jaccard distance; (d) Hellinger distance.}
\label{fig:exp1}
\end{figure}

Figure \ref{fig:exp1} reports the relations between the proposed normalized divergence and the other investigated measures. Calculations were performed by setting a number of cells equal to 5 and a total distributed quantity of 15. The experiment generated 1001 unordered distributions, of which 30 were monotonically decreasing ordered. Thus, a total of $1001 \time 1001$ two-by-two distribution comparisons were performed.

The proposed measure is more correlated with the non-symmetric divergence, rather than the other measures. Pearson correlation coefficient \cite{lee1988thirteen} reaches a value of 0.97 between the proposed divergence and the unnormalized one, and a correlation value of 0.96 between the proposed measure and the symmetric divergence.
The complete list of Pearson correlation coefficients between the compared measures is reported in Table 1 of Supplementary Material.

\section{Relation with distributional properties}
\label{sec:reldistrprop}
Entropic divergences, as well as other measures, can be used for prioritizing elements w.r.t. their deviance form randomness or, generically, from a background model. Thus, it can be interesting to study how the rank assigned to elements, based on their divergence, changes when the four different measures are used.
In what follows, the uniform distribution is used as background model and the measure of divergence from it is calculated for the set of ordered distributions that can be formed by taking into account the same quantity that is distributed in the uniform shape. For the experiments, a number of cells equal to 8 and a total quantity of 32 have been taken into account. In this way, the uniform distribution assigns a quantity of 4 to each cell. The difference w.r.t. the previous experiments, where 5 cells and 15 elements are considered, is because the previous experiment generates only 30 distinct ordered distributions which is a relatively small number. On the contrary, a setup with 8 cells and 32 elements generates a high number of unordered distributions (2,629,575) that leads to a huge number of two-by-two comparisons. As a pro, the new setup generates 919 ordered distributions, that can be considered a sufficient amount to draw experimental conclusions.

The correlation between the measures and the properties of the compared distributions is investigated. Entropy, coefficient of variation, skewness and Kurtosis's index are the considered properties.

Figure \ref{fig:exp2} shows the relation between the four investigated measures and the entropy of the ordered distribution that is compared with the uniform distribution. The simple Kullback-Leibler divergence is the measure which better correlates with the entropy, followed by the proposed normalized divergence. Table 2 of Supplementary Material reports the correlations between the measures and the entropy. The numeric correlations confirm what is shown by the graphics.

\begin{figure}[h]
\centering
\includegraphics[width=0.9\linewidth]{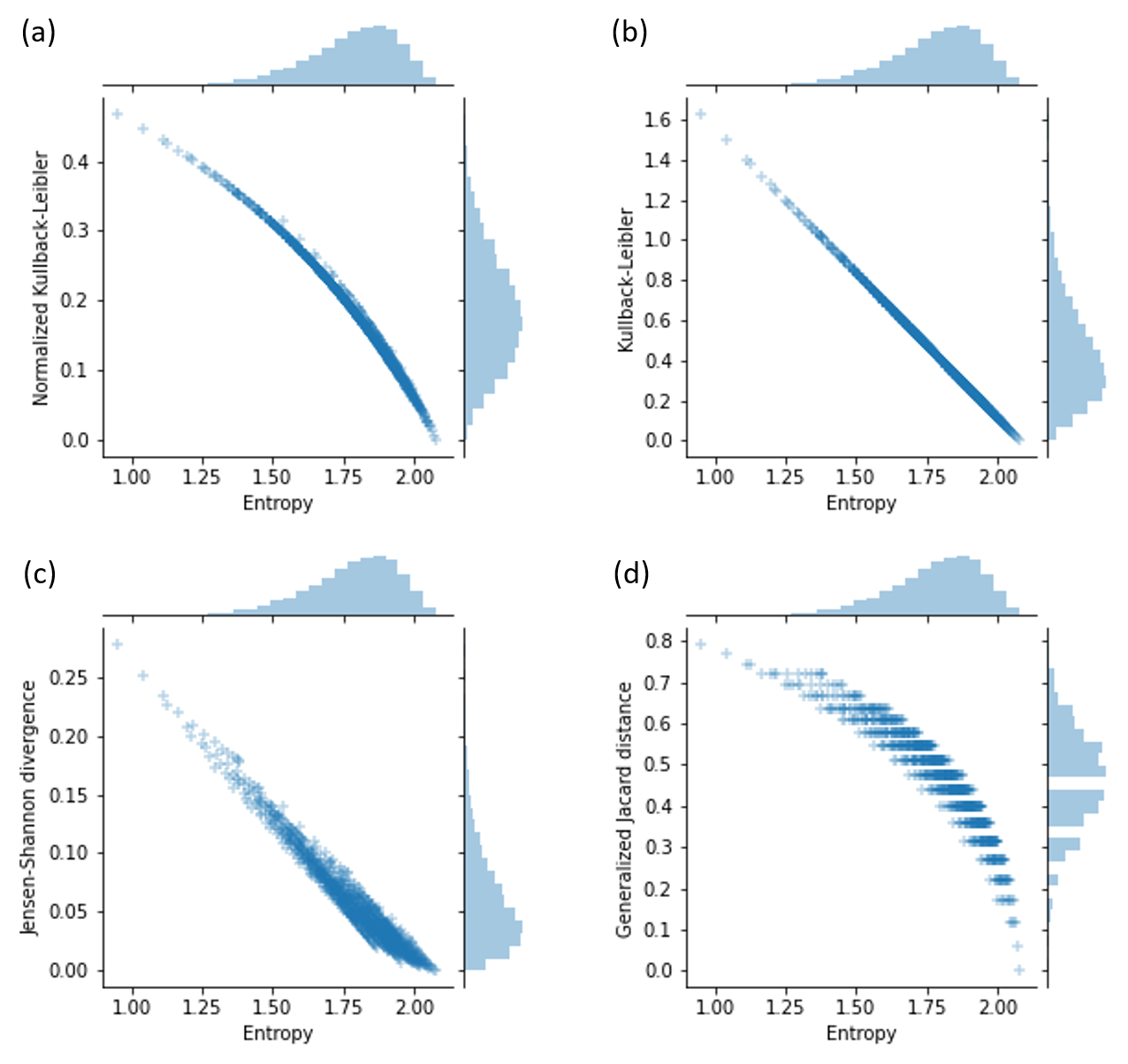}
\caption{Scatter plots generated by putting in relation four of the investigated measures and the entropy of the set of monotonically ordered distributions, generated with 8 cells and 32 dots, and the corresponding uniform distribution.}
\label{fig:exp2}
\end{figure}

Figure \ref{fig:exp3} shows the relation between the four measures and the coefficient of variation of the ordered distribution that is compared with the uniform one. Pearson correlation coefficients are reported in Table 2 of Supplementary Material. Differently from entropy-related correlations, the proposed normalized measure is the one which better correlates with the coefficient of variation, followed by the unnormalized entropic divergence. Moreover, differently from the unnormalized Kullbac-Leibler divergence and the Jensen-Shannon divergence, the proposed normalized divergence forms a sigmoid curve rather than an exponential trend.

\begin{figure}[h]
\centering
\includegraphics[width=0.9\linewidth]{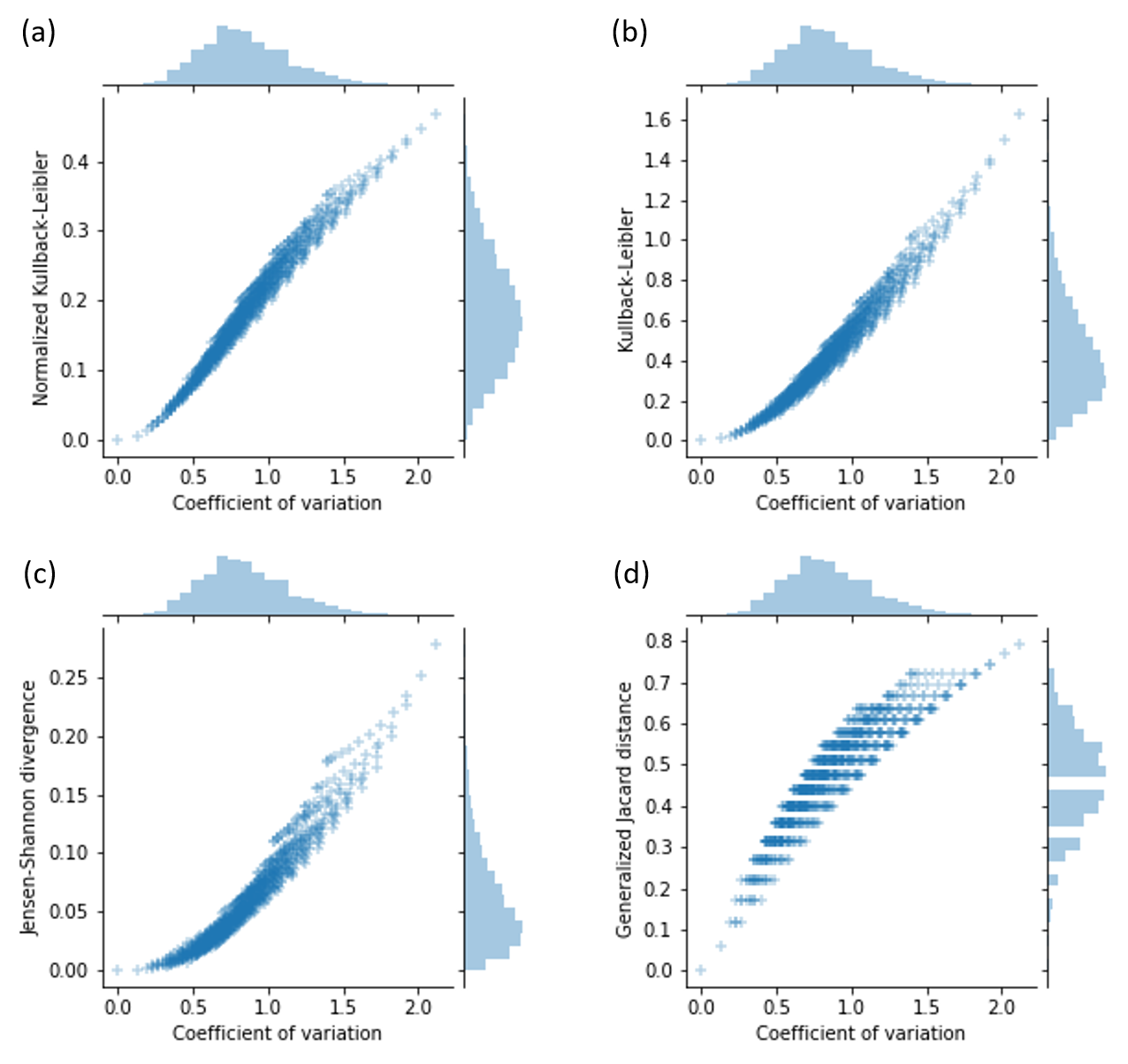}
\caption{Scatter plots generated by putting in relation four of the investigated measures and the coefficient of variation of the set of monotonically ordered distributions, generated with 8 cells and 32 dots, and the corresponding uniform distribution.}
\label{fig:exp3}
\end{figure}

Entropy and coefficient of variation are the distributional properties that better correlate with the investigated measures. In Figures 1 and 2 of Supplementary Material, it is shown that the skewness and the Kurtosis's index of the compared unordered distribution weakly correlate with the measures. However, both distributional properties form shapes similar to grids when they are plotted. This behaviour is possibly due to the discrete nature of the compared distributions. 

\section{Outcome spread diversity}
\label{sec:odiversity}
Scatter plots and histograms of the proposed figures shows interesting behaviours of the investigated measures related to how the output values of these measures spread along the output range.

For example, there are visible clusters that are formed by the generalized Jaccard distance (see Figure \ref{fig:exp1}).
This behaviour directly emerges from Equation \ref{eq:genjaccard} since the Jaccard distance tends to flatten the punctual comparison among the elements in the domain of the distribution into a sum of values of multiplicity.
The divergences seem to not produce such clusters, however, it can be helpful to investigate more properly such phenomenon.
Distances between consecutive values of the two measures has been taken into account. Given a set of $n$ comparisons, a vector of size $n$ is built from the values of the specific measure on such comparisons. 
The vector is sorted, then runs within the vector reporting the same value are substitute with one single value.
The differences between adjacent positions of the vector are extracted. Then, the mean and the standard deviation are computed. 
The elimination of the runs on the vector of the generalized Jaccard measure decreases the size of the vector from $1001\times 1001$ to 11, as it can be observed on the figure.
The distances of the generalized Jaccard measure have a mean equal to 0.08 and a standard deviation of 0.02.
On the contrary, the distances of the normalized entropic divergence have an average of 0.00004 and a standard deviation of 0.0005. Thus, it seems that the divergence is not forming clusters.
\\

Regarding experiments presented in Section \ref{sec:reldistrprop}, the compared measures have different output ranges. 
The unnormalized entropic divergence and the Jensen-Shannon divergence have no upper-bound by definition, on the contrary, it is expected that the proposed measure and the generalized Jaccard distance should range between 0 and 1.
The proposed normalized divergence ranges from 0 to circa 0.5 because one of the two compared distribution is always the uniform distribution. In fact, the monotonically ordered distribution that more diverges from the uniform distribution is the one which assigns all the available quantity to the first cell. Such a distribution is completely opposed to $U$ and the uniform distribution is exactly in the middle of them. Thus the divergence from the distribution to the uniform one is half of the divergence from $U$.
The generalized Jaccard distance is influenced by the fact that values closed to 1 can not be reached because the compared distributions have no term equal to 0. In fact, the maximum observable distance is 0.8.

Table \ref{tab:valreanges} shows the maximum value that each investigated measure reaches at varying the number of cells and dots with which distributions are built from. All the measures have a minimum value of 0 because the uniform distribution is among the distributions that are compared to itself.
The proposed normalized divergence takes values that are closed to 0.5 but never equal to such a value. The reason resides in the discretized nature of the compared distributions. However, some pattern emerges from the table. The values of the measures are directly related to the number of dots that are distributed. The smaller is the number of dots, the higher is the value of the proposed normalized measure. This behaviour is opposite to one of the other three measures which increase their value on increasing the number of distributed dots.
Intuitively, the distribution which maximizes the divergence/distance from the uniform distribution is the one which assigns all the available dots to the first cell, thus it is specular to $U$. This intuition is also confirmed by computational experiments.
The fact that the measure takes different values depends on the ratio between the dots that are assigned to the first cell and the number of cells. For example, the uniform distributions obtained for 6 cells and 12 dots and for 7 cells and 14 dots are almost the same. Both of them assign 2 dots to each cell. However, the number of available dots, after assigning one dots to each cell, is 6 in the first case and 7 in the second case. Thus the difference between the two generalized Jaccard distance is $\frac{2}{7}$ versus $\frac{2}{8}+\frac{1}{2} = \frac{3}{4}$ because, except for the first cell, all the other cell carry a value of $\frac{1}{2}$  for both configurations, and the configuration with 7 cells has an additional cell
This difference, notably, leads to a different resulting value. Similar considerations can be made for the other measures.

\begin{table}
\scriptsize
    \centering
    \begin{tabular}{l|l|c|c|c|c|c}
\hline
Cells & Dots & Norm. KL & Unnorm. KL & Jensen-Shannon & Hellinger & Gen. Jaccard\\
\hline
6	&	12	&	0.5078	&	0.6376	&	0.1395	&	0.0989	&	0.5882\\
6	&	18	&	0.4498	&	1.0876	&	0.2399	&	0.1719	&	0.7143\\
6	&	24	&	0.4297	&	1.3629	&	0.3046	&	0.2201	&	0.7692\\
6	&	30	&	0.4164	&	1.5480	&	0.3500	&	0.2546	&	0.8000\\
7	&	14	&	0.5151	&	0.7143	&	0.1518	&	0.1082	&	0.6000\\
7	&	21	&	0.4687	&	1.2057	&	0.2578	&	0.1857	&	0.7273\\
7	&	28	&	0.4502	&	1.5038	&	0.3257	&	0.2364	&	0.7826\\
7	&	35	&	0.4374	&	1.7033	&	0.3731	&	0.2726	&	0.8136\\
8	&	16	&	0.5233	&	0.7831	&	0.1622	&	0.1161	&	0.6087\\
8	&	24	&	0.4845	&	1.3103	&	0.2727	&	0.1973	&	0.7368\\
8	&	32	&	0.4672	&	1.6280	&	0.3429	&	0.2500	&	0.7925\\
8	&	40	&	0.4546	&	1.8397	&	0.3919	&	0.2876	&	0.8235\\
9	&	18	&	0.5315	&	0.8455	&	0.1711	&	0.1230	&	0.6154\\
9	&	27	&	0.4981	&	1.4043	&	0.2851	&	0.2072	&	0.7442\\
9	&	36	&	0.4815	&	1.7391	&	0.3573	&	0.2616	&	0.8000\\
9	&	45	&	0.4691	&	1.9614	&	0.4076	&	0.3002	&	0.8312\\
10	&	20	&	0.5394	&	0.9027	&	0.1789	&	0.1291	&	0.6207\\
10	&	30	&	0.5098	&	1.4897	&	0.2958	&	0.2158	&	0.7500\\
10	&	40	&	0.4938	&	1.8395	&	0.3696	&	0.2716	&	0.8060\\
10	&	50	&	0.4815	&	2.0713	&	0.4208	&	0.3112	&	0.8372\\
\hline
    \end{tabular}
    \caption{Maximum values of the five investigated measure by varying number of cells and dots which distributions are formed by.}
    \label{tab:valreanges}
\end{table}

The difference in how the measures spread the values along with the range from 0 to the maximum value is summarized in Table \ref{tab:meanmax}. 
Each experiment regards a specific number of cells and dots, as for the previous analysis. As a measure of spread, the average value divided by the maximum value is used. The closest to 0.5 is the resultant measurement, the more spread the values are. On the contrary, if the measurement tends to 0, then the values are more concentrated towards the 0, and, similarly, they are concentrated towards the maximum if the measurement tends to 1. 
The proposed normalized divergence is the one which better tends to 0.5 with an average value of 0.4296 along with the complete set of experiments.
The unnormalized KL tends to 0 more than the Jensen-Shannon divergence, that is in contrast with the mode observed in the figures, and the generalized Jaccard distance tends more to the maximum value with an average of 0.6.

\begin{table}
\scriptsize
    \centering
    \begin{tabular}{l|l|c|c|c|c|c}
\hline
Cells & Dots & Norm. KL & Unnorm. KL & Jensen-Shannon & Hellinger & Gen. Jaccard\\
\hline
6	&	12	&	0.5301	&	0.4089	&	0.4418	&	0.4379	&	0.6203\\
6	&	18	&	0.4403	&	0.3217	&	0.3540	&	0.3495	&	0.5979\\
6	&	24	&	0.3987	&	0.2865	&	0.3159	&	0.3110	&	0.5848\\
6	&	30	&	0.3739	&	0.2672	&	0.2941	&	0.2886	&	0.5766\\
7	&	14	&	0.5277	&	0.3987	&	0.4365	&	0.4315	&	0.6277\\
7	&	21	&	0.4396	&	0.3139	&	0.3523	&	0.3466	&	0.6069\\
7	&	28	&	0.3965	&	0.2778	&	0.3131	&	0.3071	&	0.5910\\
7	&	35	&	0.3709	&	0.2583	&	0.2910	&	0.2846	&	0.5815\\
8	&	16	&	0.5318	&	0.3931	&	0.4379	&	0.4314	&	0.6483\\
8	&	24	&	0.4390	&	0.3066	&	0.3505	&	0.3436	&	0.6144\\
8	&	32	&	0.3956	&	0.2711	&	0.3117	&	0.3046	&	0.5967\\
8	&	40	&	0.3694	&	0.2514	&	0.2891	&	0.2818	&	0.5860\\
9	&	18	&	0.5332	&	0.3871	&	0.4369	&	0.4292	&	0.6578\\
9	&	27	&	0.4404	&	0.3017	&	0.3508	&	0.3427	&	0.6218\\
9	&	36	&	0.3961	&	0.2660	&	0.3114	&	0.3033	&	0.6022\\
9	&	45	&	0.3692	&	0.2462	&	0.2885	&	0.2803	&	0.5906\\
10	&	20	&	0.5362	&	0.3832	&	0.4384	&	0.4294	&	0.6708\\
10	&	30	&	0.4421	&	0.2977	&	0.3514	&	0.3423	&	0.6284\\
10	&	40	&	0.3972	&	0.2621	&	0.3119	&	0.3029	&	0.6076\\
10	&	50	&	0.3696	&	0.2422	&	0.2886	&	0.2796	&	0.5951\\
\hline
avg	&		&	0.4349	&	0.3071	&	0.3483	&	0.3414	&	0.6103\\
\hline
    \end{tabular}
    \caption{Average divided by maximum value of the five investigated measures by varying number of cells and dots with which distributions are formed by.}
    \label{tab:meanmax}
\end{table}

\section{Differences in ranking outcomes}
\label{sec:diffrank}

Lastly, the difference in the ranking produced by the four measures has been investigated.
Experimental results were obtained by using 8 cells and 32 dots. The uniform distribution was compared to the set of monotonically decreasing ordered distributions, as for the previous experiment.
Then, distributions were ranked depending on the value each measure assigned to them.
Figure \ref{fig:exp6} show the comparison between the normalized entropic divergence and the three other measures in assigning the rank to the distributions. 
Each point, in one of the three plots, is a given distribution which coordinates, in the Cartesian plane, are given by the rank assigned by the two compared measures.
These charts give an idea of how different a ranking can be when different measures are applied. 
A mathematical way for comparing rankings is the Spearman's rank correlation coefficient \cite{daniel1978applied}, which values are reported in Table 3 of Supplementary Material.
The reported correlations may appear significantly high, however, there is a discordance between the measures from circa 0.05 to 0.001, which means that from 5\%  to 0.1\% of the elements are ranked differently. Such a difference may, for example, lead to different empirical p-values, which may change the results of a study.

\begin{figure}[h]
\centering
\includegraphics[width=0.9\linewidth]{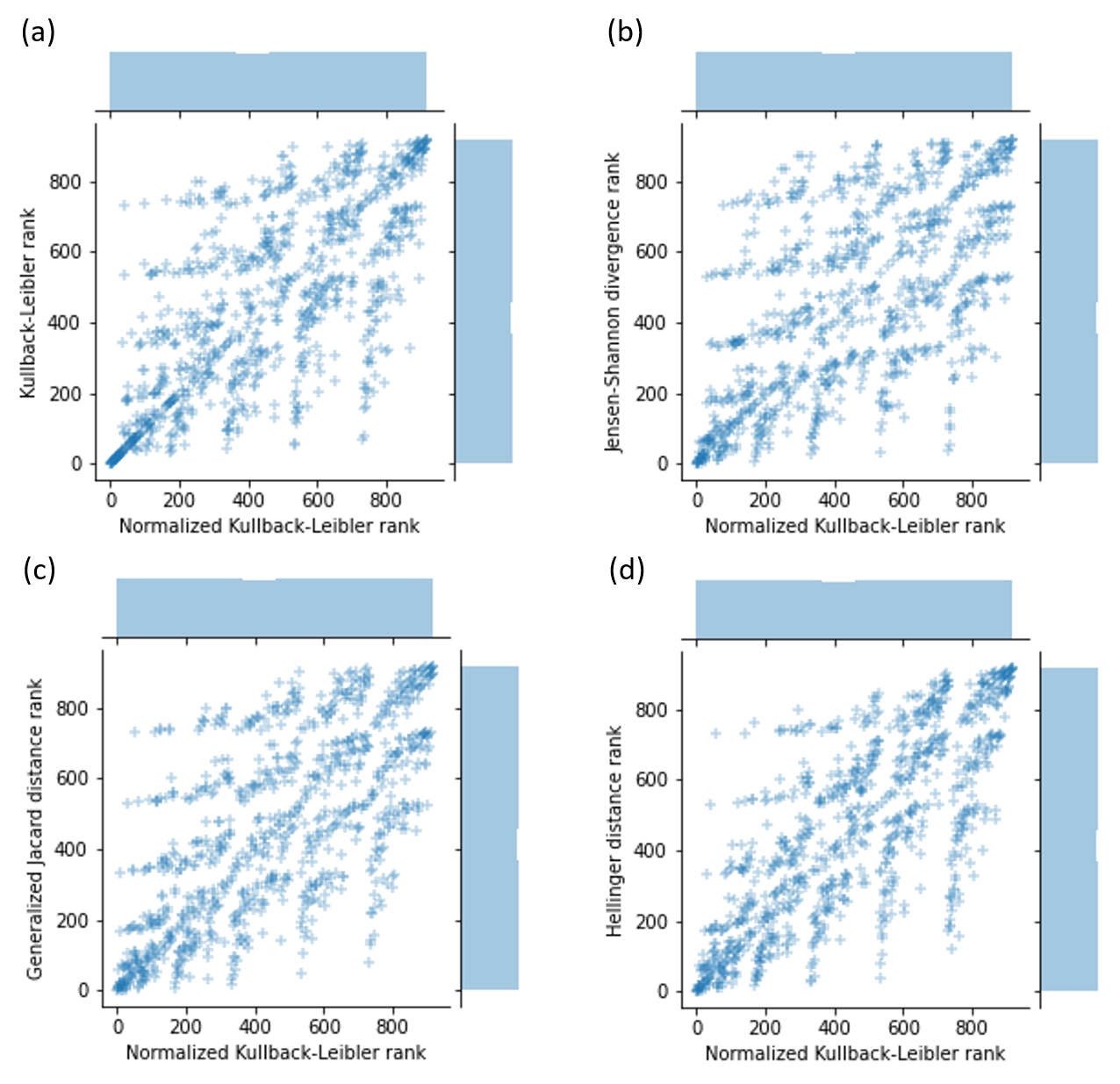}
\caption{Scatter plots obtained by taking into account the rank assigned by the proposed normalized Kullback-Leibler and the other investigated measures. The complete set of monotonically ordered distributions generated with 8 cells and 32 dots was used for extracting the rankings.}
\label{fig:exp6}
\end{figure}

\section{Conclusion}
This study shows that given a probability distribution $P$, there exists another distribution $U$ that maximizes the entropic divergence form $P$, if infinite divergences are avoided. 
$P$ and $U$ must have been generated by distributing a given discrete quantity.
If we think about quantum theory, the real word is made of discretized quantities called quanta (singular quantum). 
Thus, quantum probability distributions are in their essence multiplicity distributions. This consideration implies that the applicability of the findings presented in the current study can be of a wide range of applications.

The shape of the distribution $U$ is here characterized, and it is used for providing a notion of entropic divergence that is normalized between 0 and 1.
Empirical evaluation of such a normalized divergence w.r.t. other commonly used measures are reported. 
The evaluation shows that the proposed divergence has its specific behaviour, on varying the properties of the compared distributions, that differ from already known measures.

This study highlights an important aspect regarding the entropic divergence. An upper-bound to the divergence is obtainable only if the two compared distributions are formed by the same quantum. Future developments of divergence should take into account this aspect.


\end{document}